\documentclass{edm_article}

\usepackage{microtype}
\usepackage{graphicx}
\usepackage{subcaption}
\usepackage{booktabs}
\usepackage{color}
\usepackage{url}
\usepackage{multirow}

\usepackage{hyperref}
\graphicspath{ {./images/} }
\usepackage{amsmath,amsfonts,bm}

\def\eqref#1{equation~\ref{#1}}

\def\1{\bm{1}}

\def\eps{{\epsilon}}

\def\va{{\bm{a}}}

\def\vd{{\bm{d}}}

\def\vk{{\bm{k}}}

\def\vr{{\bm{r}}}
\def\vs{{\bm{s}}}

\def\vu{{\bm{u}}}

\def\vx{{\bm{x}}}

\def\vz{{\bm{z}}}

\DeclareMathAlphabet{\mathsfit}{\encodingdefault}{\sfdefault}{m}{sl}
\SetMathAlphabet{\mathsfit}{bold}{\encodingdefault}{\sfdefault}{bx}{n}

\newcommand{\KL}{D_{\mathrm{KL}}}

\usepackage[ruled,vlined]{algorithm2e}

\newtheorem{thm}{Theorem}[section]

\hypersetup{draft}
\begin{document}


\title{Variational Item Response Theory:\\ Fast, Accurate, and Expressive}
\numberofauthors{5}
\author{
    Mike Wu$^{\blacklozenge}$, Richard L. Davis$^{\lozenge}$, Benjamin W. Domingue$^{\lozenge}$, Chris Piech$^{\blacklozenge}$, Noah Goodman$^{\blacklozenge,\blacksquare}$ \\
    \affaddr{Department of Computer Science ($\blacklozenge$), Education ($\lozenge$), and Psychology ($\blacksquare$)} \\
    \affaddr{Stanford University} \\
    \email{\texttt{\{wumike,rldavis,bdomingu,cpiech,ngoodman\}@stanford.edu}}
}

\maketitle

\begin{abstract}
Item Response Theory (IRT) is a ubiquitous model for understanding humans based on their responses to questions, used in fields as diverse as education, medicine and psychology.
Large modern datasets offer opportunities to capture more nuances in human behavior, potentially improving test scoring and better informing public policy.
Yet larger datasets pose a difficult speed / accuracy challenge to contemporary algorithms for fitting IRT models.
We introduce a variational Bayesian inference algorithm for IRT, and show that it is fast and scaleable without sacrificing accuracy.
Using this inference approach we then extend classic IRT with expressive Bayesian models of responses.
Applying this method to five large-scale item response datasets from cognitive science and education yields higher log likelihoods and improvements in imputing missing data.
The algorithm implementation is open-source, and easily usable.
\end{abstract}

\section{Introduction}

The task of estimating human ability from stochastic responses to a series of questions has been studied since the 1950s in thousands of papers spanning several fields.
The standard statistical model for this problem, Item Response Theory (IRT), is used every day around the world, in many critical contexts including college admissions tests, school-system assessment, survey analysis, popular questionnaires, and medical diagnosis.

As datasets become larger, new challenges and opportunities for improving IRT models present themselves.
On the one hand, massive datasets offer the opportunity to better understand human behavior, fitting more expressive models.
On the other hand, the algorithms that work for fitting small datasets often become intractable for larger data sizes.
Indeed, despite a large body of literature, contemporary IRT methods fall short -- it remains surprisingly difficult to estimate human ability from stochastic responses.
One crucial bottleneck is that the most accurate, state-of-the-art Bayesian inference algorithms are prohibitively slow, while faster algorithms (such as the popular maximum marginal likelihood estimators) are less accurate and poorly capture uncertainty.
This leaves practitioners with a choice: either have nuanced Bayesian models with appropriate inference or have timely computation.

In the field of artificial intelligence, a revolution in deep generative models via \emph{variational inference} \cite{kingma2013auto,rezende2014stochastic} has demonstrated an impressive ability to perform fast inference for complex Bayesian models.
In this paper, we present a novel application of variational inference to IRT, validate the resulting algorithms with synthetic datasets, and apply them to real world datasets.
We then show that this inference approach allows us to extend classic IRT response models with deep neural network components. We find that these more flexible models better fit the large real world datasets.
Specifically, our contributions are as follows:
\begin{enumerate}

    \item \textbf{Variational inference for IRT:} We derive a new optimization objective --- the Variational Item response theory Lower Bound, or VIBO --- to perform inference in IRT models.
    By learning a mapping from responses to posterior distributions over ability and item characteristics, VIBO is ``amortized" to solve inference queries efficiently.

    \item \textbf{Faster inference:}
    We find VIBO to be much faster than previous Bayesian techniques and usable on much larger datasets without loss in accuracy.

    \item \textbf{More expressive:} Our inference approach is naturally compatible with deep generative models and, as such, we enable the novel extension of Bayesian IRT models to use neural-network-based representations for inputs, predictions, and student ability. We develop the first deep generative IRT models.

    \item \textbf{Simple code:} Using our VIBO python package is only a few lines of code that is easy to extend.

    \item \textbf{Real world application: } We demonstrate the impact of faster inference and expressive models by applying our algorithms to datasets including: PISA, DuoLingo and Gradescope. We achieve up to 200 times speedup and show improved accuracy at imputing hidden responses.
    At scale, these improvements in efficiency save hundreds of hours of computation.
\end{enumerate}


\section{Background}

We briefly review several variations of item response theory and the fundamental principles of approximate Bayesian inference, focusing on modern variational inference.

\subsection{Item Response Theory}

Imagine answering a series of multiple choice questions.
For example, consider a personality survey, a homework assignment, or a school entrance examination.
Selecting a response to each question is an interaction between your ``ability'' (knowledge or features) and the characteristics of the question, such as its difficulty.
The goal in examination analysis is to gauge this unknown ability of each student and the unknown item characteristics based only on responses.
Early procedures \cite{edgeworth1888statistics}
defaulted to very simple methods, such as counting the number of correct responses, which ignore differences in question quality.
In reality, we understand that not all questions are created equal: some may be hard to understand while others may test more difficult concepts.
To capture these nuances, Item Response Theory (IRT) was developed as a mathematical framework to reason jointly about people's ability and the items themselves.

The IRT model plays an impactful role in many large institutions.
It is the preferred method for estimating ability in several state assessments in the United States, for international assessments gauging educational competency across countries \cite{harlen2001assessment}, and for the National Assessment of Educational Programs (NAEP), a large-scale measurement of literacy and reading comprehension in the US \cite{ravitch1995national}.
Beyond education, IRT is a method widely used in cognitive science and psychology, for instance with regards to studies of language acquisition and development \cite{hartshorne2018critical,magdalena2016ratings,frank2017wordbank,braginsky2015developmental}.


\begin{figure}[h!]
    \centering
    \begin{subfigure}[b]{0.3\linewidth}
        \includegraphics[width=\linewidth]{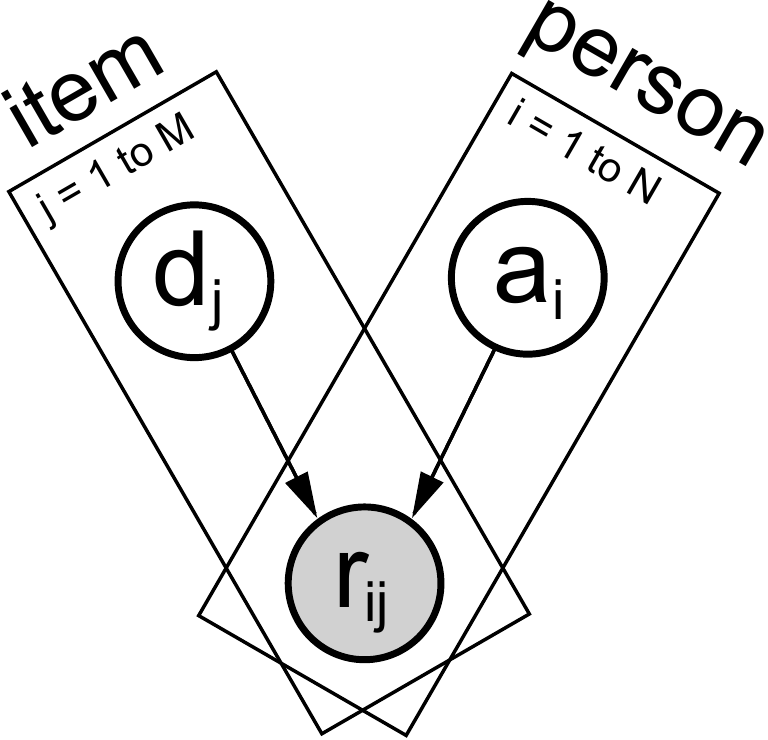}
        \caption{}
    \end{subfigure}
    \begin{subfigure}[b]{0.3\linewidth}
        \includegraphics[width=\linewidth]{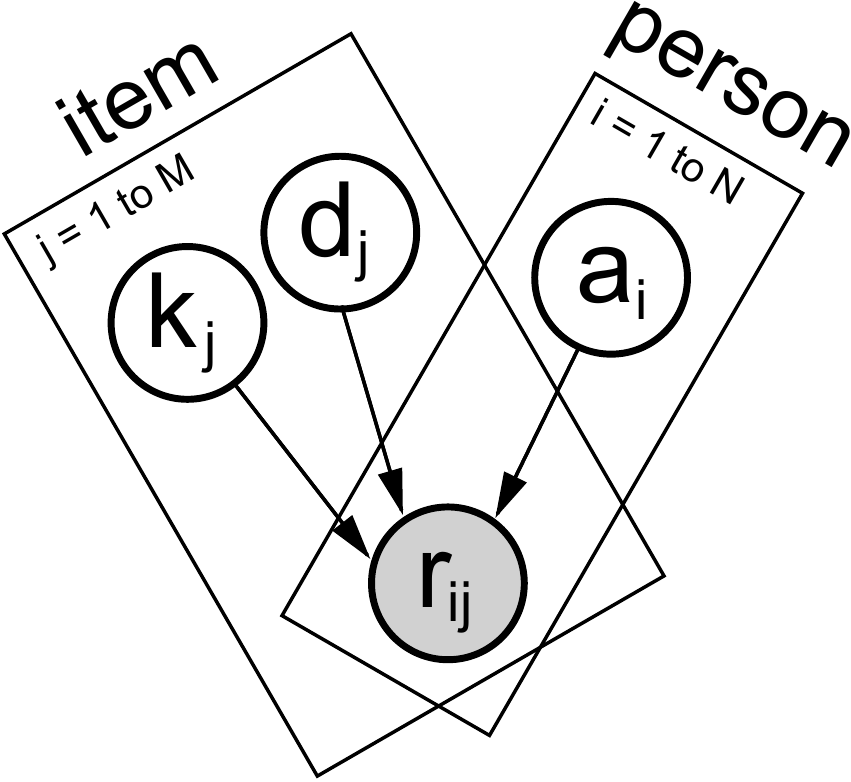}
        \caption{}
    \end{subfigure}
    \begin{subfigure}[b]{0.3\linewidth}
        \includegraphics[width=\linewidth]{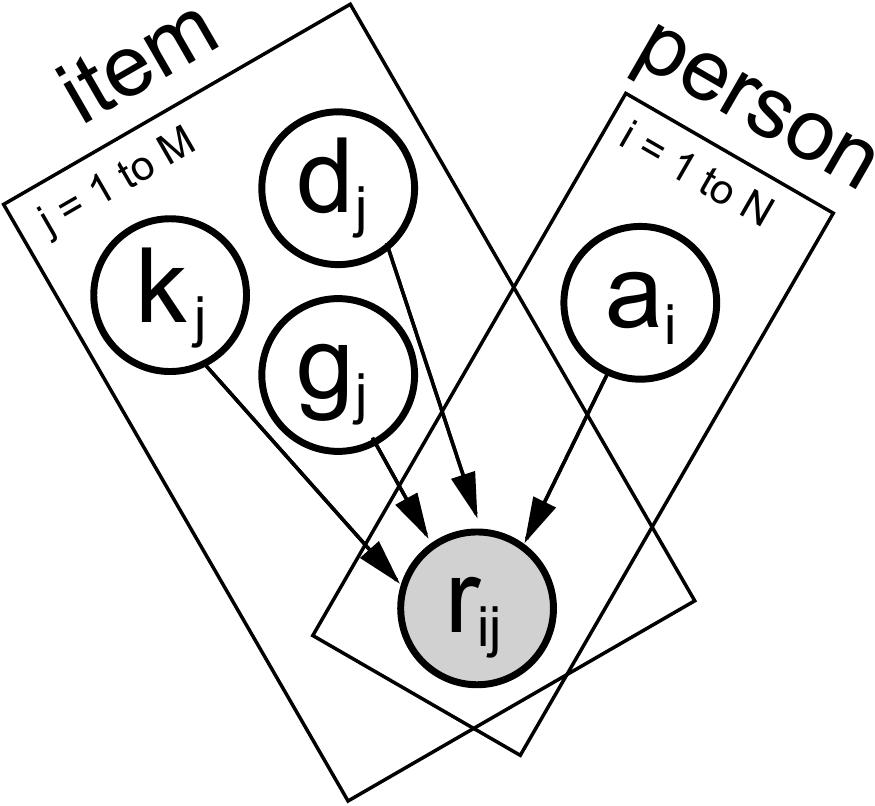}
        \caption{}
    \end{subfigure}
    \caption{Graphical models for the (a) 1PL, (b) 2PL, and (c) 3PL Item Response Theories. Observed variables are shaded. Arrows represent dependency between random variables and each rectangle represents a plate (i.e.~repeated observations).}
    \label{fig:irt_graph}
\end{figure}
IRT has many forms; we review the most standard (Fig.~\ref{fig:irt_graph}).
The simplest class of IRT summarizes the ability of a person with a single parameter.
This class contains three versions: 1PL, 2PL, and 3PL IRT, each of which differ by the number of free variables used to characterize an item.
The 1PL IRT model, also called the Rasch model \cite{rasch1960studies}, is given in Eq.~\ref{eq:1pl_irt},
\begin{equation}
    p(r_{i,j}=1|a_i,d_j) = \frac{1}{1 + e^{-(a_i - d_j)}}
\label{eq:1pl_irt}
\end{equation}
where $r_{i,j}$ is the response by the $i$-th person to the $j$-th item.
There are $N$ people and $M$ items in total.
Each item in the 1PL model is characterized by a single number representing difficulty, $d_j$.
As the 1PL model is equivalent to a logistic function, a higher difficulty requires a higher ability in order to respond correctly.
Next, the 2PL IRT model\footnote{We default to 2PL as the  pseudo-guessing parameter introduces several invariances in the model. This requires far more data to infer ability accurately, as measured by our own synthetic experiments.
For practitioners, we warn against using 3PL for small to medium datasets.} adds a \textit{discrimination} parameter, $k_j$ for each item that controls the slope (or scale) of the logistic curve.
We can expect items with higher discrimination to more quickly separate people of low and high ability.
The 3PL IRT model further adds a \textit{pseudo-guessing} parameter, $g_j$ for each item that sets the asymptotic minimum of the logistic curve.
We can interpret pseudo-guessing as the probability of success if the respondent were to make a reasonable guess on an item.
The 2PL and 3PL IRT models are:
\begin{equation}
   p(r_{i,j}|a_i, \vd_{j}) = \frac{1}{1 + e^{-k_j a_i - d_j}}\enspace \textup{or}\enspace  g_j + \frac{1 - g_j}{1 + e^{-k_j a_i - d_j}}
\label{eq:2and3pl_irt}
\end{equation}
where $\vd_j = \{k_j, d_j\}$ for 2PL and $\vd_j = \{g_j, k_j, d_j\}$ for 3PL.
See Fig.~\ref{fig:irt_graph} for graphical models of each of these IRT models.

A single ability dimension is sometimes insufficient to capture the relevant variation in human responses.
For instance, if we are measuring a person's understanding on elementary arithmetic, then a single dimension may suffice in capturing the majority of the variance.
However, if we are instead measuring a person's general mathematics ability, a single real number no longer seems sufficient.
Even if we bound the domain to middle school mathematics, there are several factors that contribute to ``mathematical understanding" (e.g. proficiency in algebra versus geometry).
Summarizing a person with a single number in this setting would result in a fairly loose approximation.
For cases where multiple facets of ability contribute to performance, we consider \textit{multidimensional} item response theory \cite{ackerman1994using,reckase2009multidimensional,mcdonald2000basis}.
We focus on 2PL multidimensional IRT (MIRT):
\begin{equation}
    p(r_{i,j}=1|\va_i,\vk_j,d_j) = \frac{1}{1 + e^{-\mathbf{a}_i^T \vk_j - d_j}}
\label{eq:3pl_mirt}
\end{equation}
where we use bolded notation $\va_i = (a^{(1)}_i, a^{(2)}_i, \ldots a^{(K)}_i)$ to represent a $K$ dimensional vector.
Notice that the item discrimination becomes a vector of equal size to ability.


In practice, given a (possibly incomplete) $N \times M$ matrix of observed responses, we want to \textit{infer} the ability of all $N$ people and the characteristics of all $M$ items.
Next, we provide a brief overview of inference in IRT.

\subsection{Inference in Item Response Theory}

We compare and contrast the three popular methods used to perform inference for IRT in research and industry. Inference algorithms are critical for item response theory as slow or inaccurate algorithms prevent the use of appropriate models.

\textbf{Maximum Likelihood Estimation}$\quad$
A straightforward approach is to pick the most likely ability and item features given the observed responses.
To do so we optimize:
\begin{align}
    \mathcal{L}_{\textup{MLE}} &= \max_{\{\va_i\}_{i=1}^N, \{\vd_j\}_{j=1}^M} \sum_{i=1}^N \sum_{j=1}^M \log p(r_{ij}|\va_i, \vd_j)
    \label{eqn:jmle}
\end{align}
with stochastic gradient descent (SGD). The symbol $\vd_j$ represents all item features e.g. $\vd_j = \{ d_j, \vk_j \}$ for 2PL.
Eq.~\ref{eqn:jmle} is often called the Joint Maximum Likelihood Estimator \cite{beguin2001mcmc,embretson2013item}, abbreviated MLE.
MLE poses inference as a supervised regression problem in which we choose the most likely unknown variables to match known dependent variables.
While MLE is simple to understand and implement, it lacks any measure of uncertainty; this can have important consequences especially when responses are missing.

\textbf{Expectation Maximization}$\quad$
Several papers have pointed out that when using MLE, the number of unknown parameters increases with the number of people \cite{bock1981marginal,haberman1977maximum}.
In particular, \cite{haberman1977maximum} shows that in practical settings with a finite number of items, standard convergence theorems do not hold for MLE as the number of people grows.
To remedy this, the authors instead treat ability as a nuisance parameter and marginalized it out \cite{bock1981marginal, bock1988full}.
Brock et. al. introduces an Expectation-Maximization (EM) \cite{dempster1977maximum} algorithm to iterate between (1) updating beliefs about item characteristics and (2) using the updated beliefs to define a marginal distribution (without ability) $p(r_{ij}|\vd_{j})$ by numerical integration of $\va_i$.
Appropriately, this algorithm is referred to as Maximum Marginal Likelihood Estimation, which we abbreviate as EM.
Eq.~\ref{eqn:em} shows the E and M steps for EM.
\begin{align}
    \textup{E step}:&\quad p(r_{ij}|\vd^{(t)}_j) = \int_{\va_i} p(r_{ij}|\va_i, \vd^{(t)}_j)p(\va_i) d \va_i \\
    \textup{M step}:&\quad \vd_j^{(t+1)} = \arg\max_{\vd_j} \sum_{i=1}^N \log p(r_{ij} |\vd^{(t)}_j)
    \label{eqn:em}
\end{align}
where $(t)$ represents the iteration count.
We often choose $p(\va_i)$ to be a simple prior distribution like standard Normal.
In general, the integral in the E-step is intractable:
EM uses a Gaussian-Hermite quadrature to discretely approximate $p(r_{ij}|\vd^{(t)}_j)$.
See \cite{harwell1988item} for a closed form expression for $\vd_j^{(t+1)}$ in the M step.
This method finds the maximum a posteriori (MAP) estimate for item characteristics.
EM does not infer ability as it is ``ignored" in the model: the common workaround is to use EM to infer item characteristics, then fit ability using a second auxiliary model.
In practice, EM has grown to be ubiquitous in industry as it is incredibly fast for small to medium sized datasets.
However, we expect that EM may scale poorly to large datasets and higher dimensions as numerical integration requires far more points to properly measure a high dimensional volume.

\textbf{Hamiltonian Monte Carlo}$\quad$
The two inference methods above give only point estimates for ability and item characteristics.
In contrast Bayesian approaches seek to capture the true posterior over ability and item characteristics given observed responses, $p(\va_i,\vd_{1:M}|\vr_{i,1:M})$ where $\vr_{i,1:M} = (r_{i,1}, \cdots, r_{i,M})$.
Doing so provides estimates of uncertainty and characterizes features of the joint distribution that cannot be represented in point estimates, such as multimodality and parameter correlation.
In practice, this can be very useful for a holistic and robust understanding of student ability.

The common technique for Bayesian estimation in IRT uses Markov Chain Monte Carlo (MCMC) \cite{hastings1970monte,gelfand1990sampling} to draw samples from the posterior by constructing a Markov chain carefully designed such that $p(\va_i,\vd_{1:M}|\vr_{i,1:M})$ is the equilibrium distribution.
By running the chain longer, we can closely match the distribution of drawn samples to the true posterior.
Hamiltonian Monte Carlo (HMC) \cite{neal2011mcmc,neal1994improved,hoffman2014no} is an efficient version of MCMC for continuous state spaces.
We recommend \cite{hoffman2014no} for a good review of HMC.

The strength of this approach is that the samples generated capture the true posterior (if the algorithm is run long enough).
However the computational costs for MCMC can be very high, and the cost scales at least linearly with the number of latent parameters --- which for IRT is directly related to data size.
With new datasets of millions of observations, such limitations can be debilitating.
Fortunately, there exist a second class of approximate Bayesian techniques that have gained significant attention lately in the machine learning community.
Next, we provide a careful review of \textit{variational inference}.

\subsection{Variational Methods}
The main intuition of variational inference (VI) is to treat inference as an optimization problem: starting with a family of distributions, the goal is to pick the one that best approximates the true posterior, by minimizing an estimate of the mismatch between true and approximate distributions.
We will first describe VI in the general context of a latent variable model, and then will apply VI to IRT.

Let $\vx \in \mathcal{X}$ and $\vz \in \mathcal{Z}$ represent observed and latent variables, respectively.
In VI \cite{jordan1999introduction,wainwright2008graphical,blei2017variational}, we introduce a family of tractable distributions
over $\vz$ (such that we can easily sample and score).
We wish to find the member $q_{\psi^*(\vx)} \in \mathcal{Q}$ that minimizes the Kullback-Leibler (KL) divergence between itself and the true posterior:
\begin{equation}
q_{\psi^*(\vx)}(\vz) = \arg \min_{q_{\psi(\vx)}} \KL(q_{\psi(\vx)}(\vz) || p(\vz|\vx))
\label{eq:vi:unamortized}
\end{equation}
where $\psi(\vx)$ are parameters that define each distribution.
For example, $\psi(\vx)$ would be the mean and scale for a Gaussian distribution.
Since the ``best" approximate posterior $q_{\psi^*(\vx)}$ depends on the observed variables, its parameters have $\vx$ as a dependent variable.
To be clear, there is one approximate posterior for every possible value of the observed variables.

Frequently, we need to do inference for many different values of the observed variables $\vx$. Let $p_{\mathcal{D}}(\vx)$ be an empirical distribution over the observed variables, which is equivalent to the marginal $p(\vx)$ if the generative model is correctly specified. Then, the average quality of the variational approximations is measured by
\begin{equation}
    \mathbb{E}_{p_{\mathcal{D}}(\vx)}\left[ \max_{\psi(\vx)} \mathbb{E}_{q_{\psi(\vx)}(\vz)}\left[\log \frac{p(\vx,\vz)}{q_{\psi(\vx)}(\vz)}\right] \right]
    \label{eq:elbo}
\end{equation}
In practice, $p_{\mathcal{D}}(\vx)$ is unknown but we assume access to a dataset $\mathcal{D}$ of examples i.i.d. sampled from $p_{\mathcal{D}}(\vx)$; this is sufficient to evaluate Eq.~\ref{eq:elbo}.

\textbf{Amortization}$\quad$
As in Eq.~\ref{eq:elbo}, we must learn an approximate posterior for each $\vx \in \mathcal{D}$.
For a large dataset $\mathcal{D}$, this can quickly grow to be unwieldly.
One such solution to this scalability problem is \textit{amortization} \cite{gershman2014amortized}, which reframes the per-observation optimization problem as a supervised regression task.
Consider learning a single deterministic mapping $f_\phi: \mathcal{X} \rightarrow \mathcal{Q}$ to predict $\psi^*(\vx)$ or equivalently $q_{\psi^*(\vx)} \in \mathcal{Q}$ as a function of the observation $\vx$.
Often, we choose $f_\phi$ to be a conditional distribution, denoted by $q_\phi(\vz|\vx) = f_\phi(\vx)(\vz)$.

The benefit of amortization is a large reduction in computational cost: the number of parameters is vastly smaller than learning a per-observation posterior.
Additionally, if we manage to learn a good regressor, then the amortized approximate posterior $q_\phi(\vz|\vx)$ could generalize to new observations $\vx \not \in \mathcal{D}$ unseen in training.
This strength has made amortized VI popular with modern latent variable models, such as the Variational Autoencoder \cite{kingma2013auto}.

Instead of Eq.~\ref{eq:elbo}, we now optimize:
\begin{equation}
\max_\phi \mathbb{E}_{p_{\mathcal{D}}(\vx)}\left[ \mathbb{E}_{q_\phi(\vz|\vx)}\left[\log \frac{p(\vx,\vz)}{q_\phi(\vz|\vx)}\right] \right]
    \label{eq:elbo_amortize}
\end{equation}
The drawback of this approach is that it introduces an \textit{amortization gap}: since we are technically using a less flexible family of approximate distributions, the quality of approximate posteriors can be inferior.

\textbf{Model Learning}$\quad$
So far we have assumed a fixed generative model $p(\vx, \vz)$.
However, often we can only specify a family of possible models $p_\theta(\vx|\vz)$ parameterized by $\theta$.
The symmetric challenge (to approximate inference) is to choose $\theta$ whose model best explains the evidence.
Naturally, we do so by maximizing the log marginal likelihood of the data
\begin{equation}
    \log p_\theta(\vx) = \log \int_{\vz} p_\theta(\vx,\vz) d\vz
\end{equation}
Using Eq.~\ref{eq:elbo_amortize}, we derive the Evidence Lower Bound (ELBO) \cite{kingma2013auto,rezende2014stochastic} with $q_\phi(\vz|\vx)$ as our inference model
\begin{equation}
    \log p_\theta(\vx) \geq \mathbb{E}_{q_\phi(\vz|\vx)}\left[ \log \frac{p_\theta(\vx,\vz)}{q_\phi(\vz|\vx)} \right] \triangleq \textup{ELBO}
    \label{eq:elbo:vae}
\end{equation}
We can jointly optimize $\phi$ and $\theta$ to maximize the ELBO.
We have the option to parameterize $p_\theta(\vx|\vz)$ and $q_\phi(\vz|\vx)$ with deep neural networks, as is common with the VAE \cite{kingma2013auto}, yielding an extremely flexible space of distributions.

\textbf{Stochastic Gradient Estimation}$\quad$
The gradients of the ELBO (Eq.~\ref{eq:elbo:vae}) with respect to $\phi$ and $\theta$ are:
\begin{align}
    \nabla_\theta \textup{ELBO} &= \mathbb{E}_{q_\phi(\vz|\vx)}[\nabla_\theta \log p_\theta(\vx,\vz)]]\label{eq:grad:theta} \\
    \nabla_\phi \textup{ELBO} &= \nabla_\phi \mathbb{E}_{q_\phi(\vz|\vx)}[ \log p_\theta(\vx,\vz)] \label{eq:grad:phi}
\end{align}
Eq.~\ref{eq:grad:theta} can be estimated using Monte Carlo samples.
However, as it stands, Eq.~\ref{eq:grad:phi} is difficult to estimate as we cannot distribute the gradient inside the inner expectation.
For certain families $\mathcal{Q}$, we can use a reparameterization trick.

\textbf{Reparameterization Estimators}$\quad$
Reparameterization is the technique of removing sampling from the gradient computation graph \cite{kingma2013auto,rezende2014stochastic}.
In particular, if we can reduce sampling $\vz \sim q_\phi(\vz|\vx)$ to sampling from a parameter-free distribution $\eps \sim p(\eps)$ plus a deterministic function application, $\vz = g_\phi(\eps)$, then we may rewrite Eq.~\ref{eq:grad:phi} as:
\begin{equation}
\nabla_\phi \textup{ELBO} = \mathbb{E}_{p(\eps)} [ \nabla_{\vz} \log \frac{p_\theta(\vx,\vz(\eps))}{q_\phi(\vz(\eps)|\vs)} \nabla_\phi g_\phi(\eps) ]
\label{eq:grad:elbo}
\end{equation}
which now can be estimated efficiently by Monte Carlo (the gradient is inside the expectation).
A benefit of reparameterization over alternative estimators (e.g. score estimator \cite{mnih2014neural} or REINFORCE \cite{williams1992simple}) is lower variance while remaining unbiased.
A common example is if $q_\phi(\vz|\vx)$ is Gaussian $\mathcal{N}(\mu,\sigma^2)$ and we choose $p(\eps)$ to be $\mathcal{N}(0, 1)$, then $g(\eps) = \eps * \sigma + \mu$.

\section{The VIBO Algorithm}
\label{sec:methods}
Having reviewed the major principles of VI, we will adapt them to IRT.
We call the resulting algorithm VIBO since it is a \textbf{V}ariational approach for \textbf{I}tem response theory based on a novel lower \textbf{BO}und.
We state and prove VIBO in the following theorem.
\begin{thm}
    Let $\va_i$ be the ability for person $i \in [1, N]$ and $\vd_{j}$ be the characteristics for item $j \in [1, M]$. We use the shorthand notation $\vd_{1:M} = (\vd_1, \ldots, \vd_M)$. Let $r_{i,j}$ be the binary response for item $j$ by person $i$. We write $\vr_{i, 1:M} = (r_{i,1}, \ldots r_{i,M})$.
    If we define the VIBO objective as:
    \begin{equation*}
        \textup{VIBO} \triangleq \mathcal{L}_{\text{recon}} + \mathbb{E}_{q_\phi(\vd_{1:M}|\vr_{i,1:M})}[D_{\text{ability}}] +D_{\text{item}}
    \end{equation*}
    where
    \begin{align*}
        \mathcal{L}_{\text{recon}} &= \mathbb{E}_{q_\phi(\va_i, \vd_{1:M}|\vr_{i,1:M})}\left[ \log p_\theta(\vr_{i,1:M}|\va_i, \vd_{1:M}) \right] \\
        D_{\text{ability}} &= \KL(q_\phi(\va_i|\vd_{1:M},\vr_{i,1:M}) || p(\vu_i)) \\
        D_{\text{item}} &= \KL(q_\phi(\vd_{1:M}|\vr_{i,1:M})||p(\vd_{1:M}))
    \end{align*}
    and assume the joint posterior factors as follows
    \begin{equation*}
        q_\phi(\va_i, \vd_{1:M}|\vr_{i,1:M}) = q_\phi(\va_i|\vd_{1:M},\vr_{i,1:M})q_\phi(\vd_{1:M}|\vr_{i,1:M})
    \end{equation*}
    then $\log p(\vr_{i,1:M}) \geq \textup{VIBO}$. In othe words, VIBO is a lower bound on the log marginal probability of person $i$'s responses.
    \label{thm:virtu}
\end{thm}
\begin{proof}
Expand marginal and apply Jensen's inequality:
\begin{align*}
    \log p_\theta(\vr_{i,1:M}) &\geq \mathbb{E}_{q_\phi(\va_i, \vd_{1:M}|\vr_{i,1:M})}\left[ \log \frac{p_\theta(\vr_{i,1:M},\va_i, \vd_{1:M})}{q_\phi(\va_i, \vd_{1:M}|\vr_{i,1:M})} \right] \\
    &= \mathbb{E}_{q_\phi(\va_i, \vd_{1:M}|\vr_{i,1:M})}\left[ \log p_\theta(\vr_{i,1:M},\va_i, \vd_{1:M}) \right] \\
    & \quad + \mathbb{E}_{q_\phi(\va_i, \vd_{1:M}|\vr_{i,1:M})}\left[ \log \frac{p(\va_i)}{q_\phi(\va_i|\vd_{1:M},\vr_{i,1:M})} \right] \\
    & \quad + \mathbb{E}_{q_\phi(\va_i, \vd_{1:M}|\vr_{i,1:M})}\left[ \log \frac{p(\vd_{1:M})}{q_\phi(\vd_{1:M}|\vr_{i,1:M})} \right] \\
    &= \mathcal{L}_{\text{recon}} + \mathcal{L}_{\text{A}} + \mathcal{L}_{\text{B}}
\end{align*}
Rearranging the latter two terms, we find that:
\begin{align*}
\mathcal{L}_{\text{A}} &= \mathbb{E}_{q_\phi(\vd_{1:M}|\vr_{i,1:M})} \left[ \KL(q_\phi(\va_i|\vd_{1:M},\vr_{i,1:M}) || p(\va_i)) \right] \\
\mathcal{L}_{\text{B}} &= \mathbb{E}_{q_\phi(\vd_{1:M}|\vr_{i,1:M})}\left[ \log \frac{p(\vd_{1:M})}{q_\phi(\vd_{1:M}|\vr_{i,1:M})} \right] \\
&= \KL(q_\phi(\vd_{1:M}|\vr_{i,1:M}) || p(\vd_{1:M}))
\end{align*}
Since $\textup{VIBO} = \mathcal{L}_{\text{recon}} + \mathcal{L}_{\text{A}} + \mathcal{L}_{\text{B}}$ and KL divergences are non-negative, we have shown that VIBO is a lower bound on $\log p_\theta(\vr_{i,1:M})$.
\end{proof}

Thm.~\ref{thm:virtu} leaves several choices up to us, and we opt for the simplest ones.
For instance, the prior distributions are chosen to be independent standard Normal distributions: $p(\va_i) = \prod_{k=1}^K p(a_{i,k})$ and $p(\vd_{1:M}) = \prod_{j=1}^M p(\vd_{j})$ where $p(a_{i,k})$ and $p(\vd_j)$ are $\mathcal{N}(0,1)$.
Further, we found it sufficient to assume $q_\phi(\vd_{1:M}|\vr_{i,1:M}) = q_\phi(\vd_{1:M})= \prod_{j=1}^M q_\phi(\vd_{j})$ although nothing prevents the general case.
Initially, we assume the generative model, $p_\theta(\vr_{i,1:M}|\va_i,\vd_{1:M})$, to be an IRT model (thus $\theta$ is empty); later we explore generalizations.

The posterior $q_\phi(\va_i|\vd_{1:M},\vr_{i,1:M})$ needs to be robust to missing data as often not every person answers every question.
To achieve this, we explore the following family:
\begin{equation}
    q_\phi(\va_i|\vd_{1:M},\vr_{i,1:M}) = \prod_{j=1}^M q_\phi(\va_i|\vd_j,\vr_{i,j})
\end{equation}
If we assume each component $q_\phi(\va_i|\vd_j,\vr_{i,j})$ is Gaussian, then $q_\phi(\va_i|\vd_{1:M},\vr_{i,1:M})$ is Gaussian as well, being a Product-Of-Experts \cite{hinton1999products,wu2018multimodal}.
If item $j$ is missing, we replace its term in the product with the prior, $p(\va_i)$ representing no added information.
We found this design to outperform averaging over non-missing entries: $\frac{1}{M}\sum_{j=1}^M q_\phi(\va_i|\vd_j,\vr_{i,j})$.



As VIBO is a close cousin of the ELBO, we can estimate its gradients with respect to $\theta$ and $\phi$ similarly:
\begin{align*}
    \nabla_\theta \textup{VIBO} &= \nabla_\theta \mathcal{L}_{\text{recon}} \\
    &= \mathbb{E}_{q_\phi(\va_i, \vd_{1:M}|\vr_{i,1:M})}\left[ \nabla_\theta \log p_\theta(\vr_{i,1:M}|\va_i, \vd_{1:M}) \right] \\
    \nabla_\phi \textup{VIBO} &= \nabla_\phi \mathbb{E}_{q_\phi(\vd_{1:M}|\vr_{i,1:M})}[D_{\text{ability}}] + \nabla_\phi D_{\text{item}} \\
    &= \nabla_\phi \mathbb{E}_{q_\phi(\va_i,\vd_{1:M}|\vr_{i,1:M})}\left[\frac{p(\va_i)p(\vd_{1:M})}{q_\phi(\va_i,\vd_{1:M}|\vr_{i,1:M})}\right]
\end{align*}
As in Eq.~\ref{eq:grad:elbo}, we may wish to move the gradient inside the KL divergences by reparameterization to reduce variance.
To allow easy reparameterization, we define all variational distributions $q_\phi(\cdot|\cdot)$ as Normal distributions with diagonal covariance.
In practice, we find that estimating $\nabla_\theta \text{VIBO}$ and $\nabla_\phi \text{VIBO}$ with a single sample is sufficient.
With this setup, VIBO can be optimized using stochastic gradient descent to learn an amortized inference model that maximizes the marginal probability of observed data.
We summarize the required computation to calculate VIBO in Alg.~\ref{alg:vibo}.

\begin{algorithm}[h!]
\SetAlgoLined
    Assume we are given observed responses for person $i$, $\vr_{i1:M}$\;
    Compute $\mu_{\vd}, \sigma^2_{\vd} = q_\phi(\vd_{1:M})$\;
    Sample $\vd_{1:M} \sim \mathcal{N}(\mu_{\vd}, \sigma^2_{\vd})$\;
    Compute $\mu_{\va}, \sigma^2_{\va} = q_\phi(\va_i|\vd_{1:M},\vr_{i,1:M})$\;
    Sample $\va_i \sim \mathcal{N}(\mu_{\va}, \sigma^2_{\va})$\;
    Compute $\mathcal{L}_{\text{recon}} = \log p_\theta(\vr_{i,1:M}|\va_i,\vd_{1:M})$\;
    Compute $D_{\text{ability}} = \KL(\mathcal{N}(\mu_{\va}, \sigma^2_{\va})||\mathcal{N}(0,1))$\;
    Compute $D_{\text{item}} = \KL(\mathcal{N}(\mu_{\vd}, \sigma^2_{\vd})||\mathcal{N}(0,1))$\;
    Compute $\text{VIBO} = \mathcal{L}_{\text{recon}} + D_{\text{ability}} + D_{\text{item}}$
    \caption{VIBO Forward Pass}
    \label{alg:vibo}
\end{algorithm}
A public implementation of VIBO will be available in PyTorch and Pyro \cite{bingham2019pyro} along with a Python package for more practical uses modeled after the R package, MIRT \cite{chalmers2012mirt}.

\section{Datasets}
We explore one synthetic dataset, to build intuition and confirm parameter recovery, and five large scale applications of IRT to real world data, summarized in Table~\ref{table:datasets}.

\begin{table}[h!]
\caption{Dataset Statistics}
\label{table:datasets}
\begin{center}
\begin{small}
\begin{sc}
\begin{tabular}{lccc}
\toprule
& \# Persons & \# Items & Missing Data? \\
\midrule
CritLangAcq & 669498 & 95 & N \\
WordBank & 5520 & 797 & N \\
DuoLingo & 2587 & 2125 & Y \\
Gradescope & 1254 & 98 & Y \\
PISA & 519334 & 183 & Y \\
\bottomrule
\end{tabular}
\end{sc}
\end{small}
\end{center}
\vskip -0.1in
\end{table}

\textbf{Synthetic IRT}$\quad$
To sanity check that VIBO performs as well as other inference techniques, we synthetically generate a dataset of responses using a 2PL IRT model:
sample $\va_i \sim p(\va_i)$, $\vd_j \sim p(\vd_j)$.
Given ability and item characteristics, IRT-2PL determines a Bernoulli distribution over responses to item $j$ by person $i$. We sample once from this Bernoulli distribution to ``generate" an observation.
In this setting, we know the ground truth ability and item characteristics.
We vary $N$ and $M$ to explore parameter recovery.

\textbf{Second Language Acquisition}$\quad$
This dataset contains native and non-native English speakers answering questions to a grammar quiz\footnote{The quiz can be found at \url{www.gameswithwords.org}. The data is publically available at \url{osf.io/pyb8s}.}, which upon completion would return a prediction of the user's native language.
Using social media, over half a million users of varying ages and demographics completed the quiz.
Quiz questions often contain both visual and linguistic components.
For instance, a quiz question could ask the user to ``choose the image where the dog is chased by the cat" and present two images of animals where only one of image agrees with the caption.
Every response is thus binary, marked as correct or incorrect.
In total, there are 669,498 people with 95 items and no missing data.
The creators of this dataset use it to study the presence or absence of a ``critical period" for second language acquisition \cite{hartshorne2018critical}.
We will refer to this dataset as \textsc{CritLangAcq}.

\textbf{WordBank: Vocabulary Development}$\quad$
The MacArthur-Bates Communicative Development Inventories (CDIs) are a widely used metric for early language acquisition in children, testing concepts in vocabulary comprehension, production, gestures, and grammar.
The WordBank \cite{frank2017wordbank} database archives many independently collected CDI datasets across languages and research laboratories\footnote{\url{github.com/langcog/wordbankr}}.
The database consists of a matrix of people against vocabulary words where the $(i,j)$ entry is 1 if a parent reports that child $i$ has knowledge of word $j$ and 0 otherwise.
Some entries are missing due to slight variations in surveys and incomplete responses.
In total, there are 5,520 children responding to 797 items.

\textbf{DuoLingo: App-Based Language Learning}$\quad$
We examine the 2018 DuoLingo Shared Task on Second Language Acquisition Modeling (SLAM)\footnote{\url{sharedtask.duolingo.com/2018.html}} \cite{settles2018second}.
This dataset contains anonymized user data from the popular education application, DuoLingo.
In the application, users must choose the correct vocabulary word among a list of distractor words.
We focus on the subset of native English speakers learning Spanish and only consider lesson sessions.
Each user has a timeseries of responses to a list of vocabulary words, each of which is shown several times.
We repurpose this dataset for IRT: the goal being to infer the user's language proficiency from his or her errors.
As such, we average over all times a user has seen each vocabulary item.
For example, if the user was presented ``habla" 10 times and correctly identified the word 5 times, he or she would be given a response score of 0.5.
We then round
it to be 0 or 1.
We will revisit a continuous version in the discussion.
After processing, we are left with 2587 users and 2125 vocabulary words with missing data as users frequently drop out.
We ignore user and syntax features.

\textbf{Gradescope: Course Exam Data}$\quad$
Gradescope \cite{singh2017gradescope} is a popular course application that assists teachers in grading student assignments.
This dataset contains in total 105,218 reponses from 6,607 assignments in 2,748 courses and 139 schools.
All assignments are instructor-uploaded, fixed-template assignments, with at least 3 questions, with the majority being examinations.
We focus on course 102576, randomly chosen from the courses with the most students.
We remove students who did not respond to any questions.
We binarize the response scores by rounding up partial credit.
In total, there are 1254 students with 98 items, with missing entries, for course 102576.

\textbf{PISA 2015: International Assessment}$\quad$
The Programme for International Student Assessment (PISA) is an international exam that measures 15-year-old students' reading, mathematics, and science literacy every three years.
It is run by the Organization for Economic Cooperation and Development (OECD).
The OECD released anonymized data from PISA '15 for students from 80 countries and education systems\footnote{\url{oecd.org/pisa/data/2015database}}.
We focus on the science component.
Using IRT to access student performance is part of the pipeline the OECD uses to compute holistic literacy scores for different countries.
As part of our processing, we binarize responses, rounding any partial credit to 1.
In total, there are 519,334 students and 183 questions.
Not every student answers every question as many versions of the computer exam exist.

\section{Fast and Accurate Inference}
\label{inferenceexpts}
We will show that VIBO is as accurate as HMC and nearly as fast as MLE/EM, making Bayesian IRT a realistic, even preferred, option for modern applications.

\subsection{Evaluation}
We compare compute cost of VIBO to HMC, EM\footnote{We use the popular MIRT package in R for EM with 61 points for numerical integration.}, and MLE using  IRT-2PL by measuring wall-clock run time.
For HMC, we limit to drawing 200 samples with 100 warmup steps with no parallelization.
For VIBO and MLE, we use the Adam optimizer with a learning rate of 5e-3. We choose to conservatively optimize for 10k iterations to estimate cost.

However, speed only matters assuming good performance.
We use three metrics of accuracy:
(1) For the synthetic dataset, because we know the true ability, we can measure the expected correlation between it and the inferred ability under each algorithm (with the exception of EM as ability is not inferred).
A correlation of 1.0 would indicate perfect inference.
(2) The most general metric is the accuracy of imputed missing data.
We hold out 10\% of the responses, use the inferred ability and item characteristics to generate responses thereby populating missing entries, and compute prediction accuracy for held-out responses.
This metric is a good test of ``overfitting" to observed responses.
(3) In the case of fully Bayesian methods (HMC and VIBO) we can compare posterior predictive statistics \cite{sinharay2006posterior} to further test uncertainty calibration (which accuracy alone does not capture).
Recall that the posterior predictive is defined as:
\begin{equation*}
    p(\tilde{\vr}_{i, 1:M}|\vr_{i,1:M}) = \mathbb{E}_{p(\va_i,\vd_{1:M}|\vr_{i,1:M})}[p(\tilde{\vr}_{i,1:M}|\va_i,\vd_{1:M})]
\end{equation*}
For HMC, we are given samples from the posterior; for VIBO, we draw samples from the approximation $q_\phi(\va_i,\vd_{1:M}|\vr_{i,1:M})$.
Given such parameter samples We can then sample responses; 
we compare summary statistics of these response samples: the average number of items answered correctly per person and the average number of people who answered each item correctly.

\subsection{Synthetic Data Results}


\begin{figure}[tbh]
    \centering

      \includegraphics[width=\linewidth]{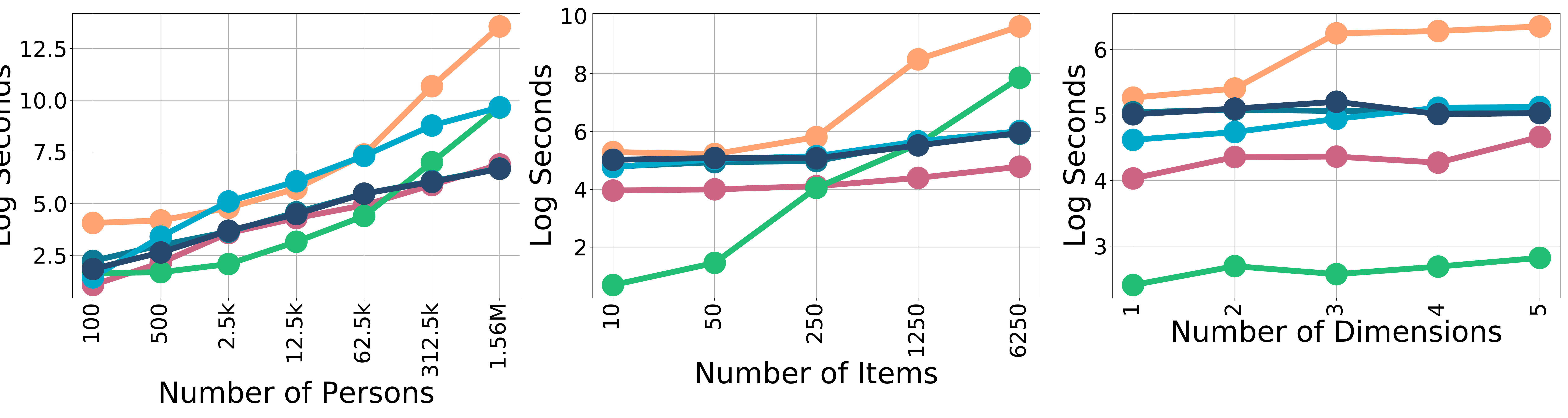}
        \includegraphics[width=\linewidth]{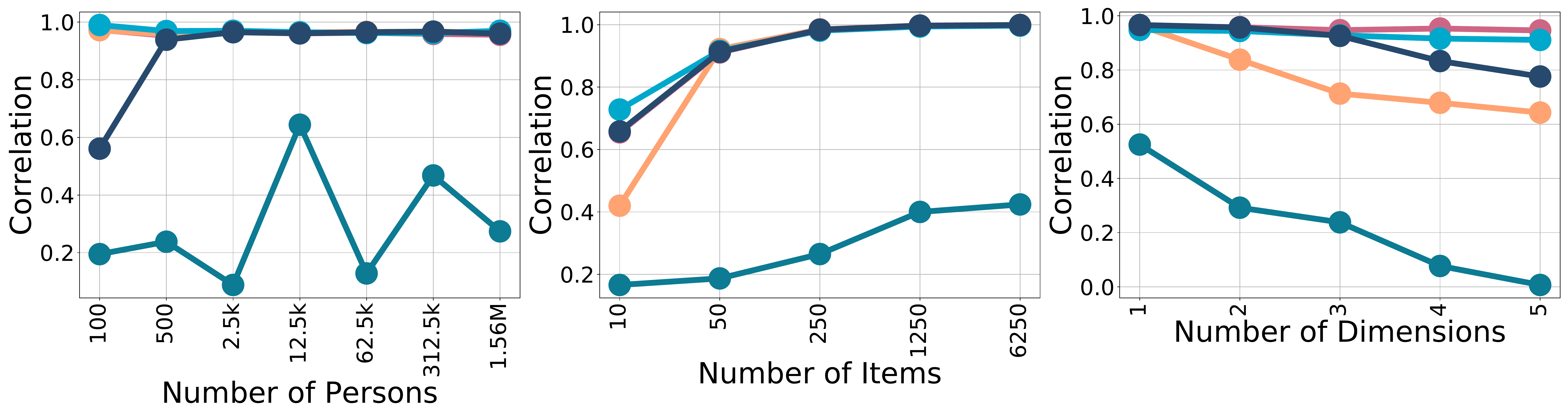}
        \includegraphics[width=\linewidth]{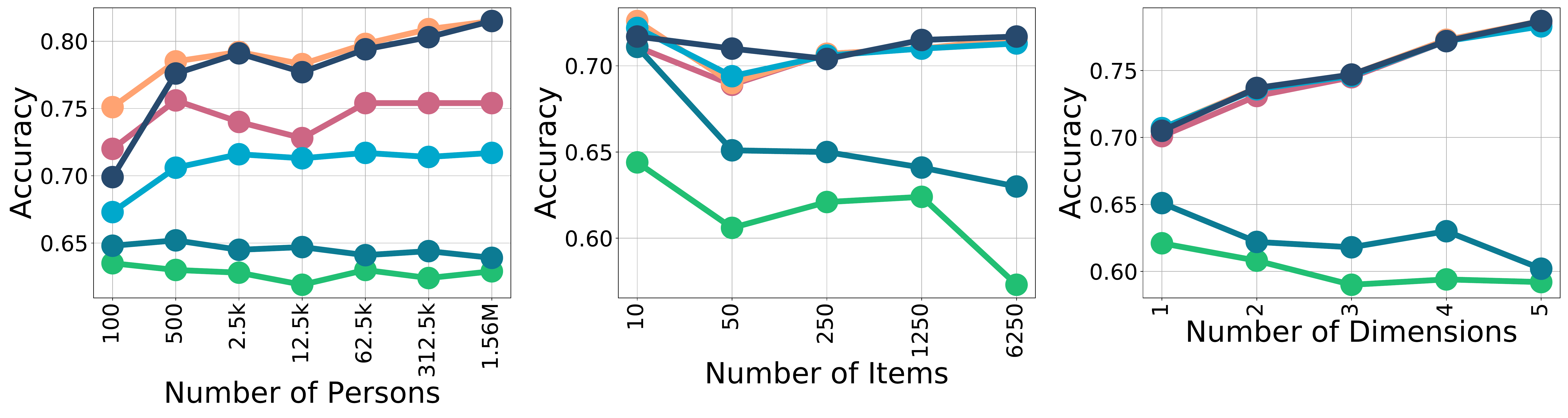}
            \includegraphics[width=0.9\linewidth]{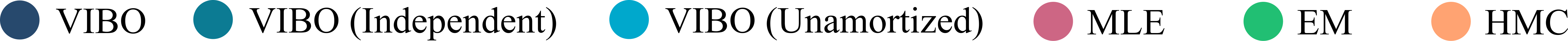}

%
%
%
    \caption{
     Performance of inference algorithms for IRT for synthetic data, as we vary the number of people, items, and latent ability dimensions.
     (Top) Computational cost in log-seconds (e.g.~1 log second is about 3 seconds whereas 10 log seconds is 6.1 hours).
    (Middle) Correlation of inferred ability with true ability (used to generate the data).
    (Bottom) Accuracy of held-out data imputation.
    }
    \label{fig:synth_results}
\end{figure}

With synthetic experiments we are free to vary $N$ and $M$ to extremes to stress test the inference algorithms: first, we range from 100 to 1.5 million people, fixing the number of items to 100 with dimensionality 1; second, we range from 10 to 6k items, fixing 10k people with dimesionality 1; third,
we vary the number of latent ability dimensions from 1 to 5, keeping a constant 10k people and 100 items.

\begin{figure*}[tbh]
    \begin{subfigure}[b]{\linewidth}
        \includegraphics[width=\linewidth]{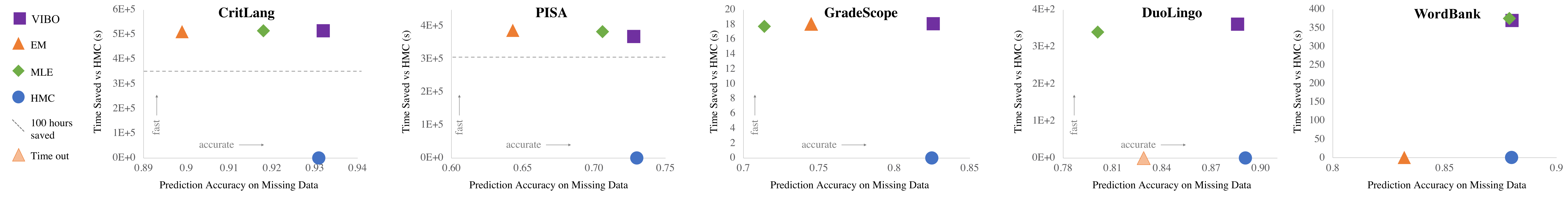}
        \caption{Real World: Accuracy of Imputing Missing Data vs Time Cost}
    \end{subfigure}
    \begin{subfigure}[b]{\linewidth}
        \includegraphics[width=\linewidth]{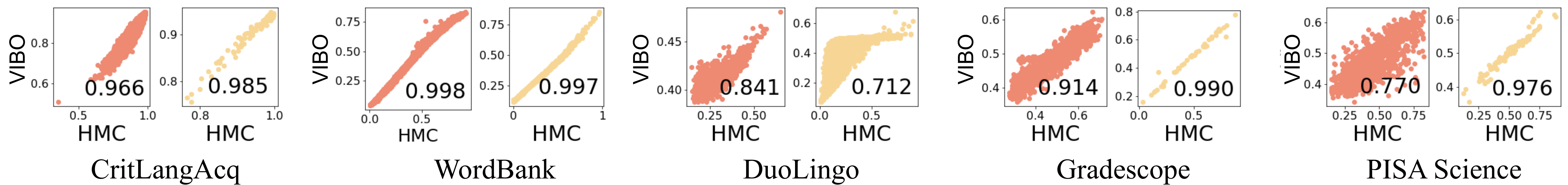}
        \caption{Real World: Posterior Predictive Checks}
    \end{subfigure}
    \caption{
    (a) Accuracy of missing data imputation for real world datasets plotted against time saved in seconds compared to using HMC.
    (b) Samples statistics from the predictive posterior defined using HMC and VIBO.
    A correlation of 1.0 would be perfect alignment between two inference techniques.
    Subfigures in red show the average number of items answered correctly for each person.
    Subfigures in yellow show the average number of people who answered each item correctly.
    }
    \label{fig:realworld}
\end{figure*}

Fig.~\ref{fig:synth_results} shows run-time and performance results of VIBO, MLE, HMC, EM, and two ablations of VIBO (discussed in Sec.~\ref{ablation}).
First, comparing parameter recovery performance (Fig.~\ref{fig:synth_results} middle) for VIBO and previous algorithms, we see that HMC, MLE and VIBO all recover parameters well.
The only notable differences are: VIBO with very few people, and HMC and (to a lesser extent) VIBO in high dimensions.
The former is because the amortized posterior approximation requires a sufficiently large dataset (around 500 people) to constrain its parameters.
The latter is a simple effect of the scaling of variance for sample-based estimates as dimensionality increases (we fixed the number of samples used, to ease speed comparisons).

Turning to the ability to predict missing data (Fig.~\ref{fig:synth_results} bottom) we see that VIBO performs equally well to HMC, except in the case of very few people (again, discussed below).
(Note that the latent dimensionality does not adversely affect VIBO or HMC for missing data prediction, because the variance is marginalized away.)
MLE also performs well as we scale number of items and latent ability dimensions, but is less able to benefit from more people.
EM on the other hand provides much worse missing data prediction in all cases.

Finally if we examine the speed of inference (Fig.~\ref{fig:synth_results} top), VIBO is only slightly slower than MLE, both of which are orders of magnitude faster than HMC.
For instance, with 1.56 million people, HMC takes 217 hours whereas VIBO takes 800 seconds.
Similarly with 6250 items, HMC takes 4.3 hours whereas VIBO takes 385 seconds.
EM is the fastest for low to medium sized datasets, though its lower accuracy makes this a dubious victory.
Furthermore, EM does not scale as well as VIBO to large datasets.


\subsection{Real World Data Results}
We next apply VIBO to real world datasets in cognitive science and education.
Fig.~\ref{fig:realworld}(a) plots the  accuracy of imputing missing data against the time saved vs HMC (the most expensive inference algorithm) for five  large-scale datasets.
Points in the upper right corner are more desirable as they are more accurate and faster.
The dotted line represents 100 hours saved compared to HMC.

From Fig.~\ref{fig:realworld}(a), we find many of the same patterns as we observed in the synthetic experiments.
Running HMC on CritLangAcq or PISA takes roughly 120 hours whereas VIBO takes 50 minutes for CritLangAcq and 5 hours for PISA, the latter being more expensive because of computation required for missing data.
In comparison, EM is at times faster than VIBO (e.g. Gradescope, PISA) and at times slower.
With respect to accuracy, VIBO and HMC are again identical, outperforming EM by up to 8\% in missing data imputation.
Interestingly, we find the ``overfitting" of MLE to be more pronounced here.
If we focus on DuoLingo and Gradescope, the two datasets with pre-existing large portions of missing values, MLE is surpassed by EM, with VIBO achieving accuracies 10\% higher.

Another way of exploring a model's ability to explain data, for fully Bayesian models, is posterior predictive checks.
Fig.~\ref{fig:realworld}(b) shows posterior predictive checks comparing VIBO and HMC.
We find that the two algorithms strongly agree about the average number of correct people and items in all datasets.
The only systematic deviations occur with DuoLingo: it is possible that this is a case where a more expressive posterior approximation would be useful in VIBO, since the number of items is greater than the number of people.

\subsection{Ablation Studies}
\label{ablation}

We compared VIBO to simpler variants that either do not amortize the posterior or do so with independent distributions of ability and item parameters. These correspond to different variational families, $\mathcal{Q}$ to choose $q$ from:
\begin{itemize}
\item VIBO (Independent): We consider the decomposition $q(\va_i, \vd_{1:M}|\vr_{i,1:M}) = q(\va_i|\vr_{i,1:M})q(\vd_{1:M})$ which treats ability and item characteristics as independent.
\item VIBO (Unamortized): We consider $q(\va_i, \vd_{1:M}|\vr_{i,1:M}) = q_{\psi(\vr_{i,1:M})}(\va_i)q(\vd_{1:M})$, which learns separate posteriors for each $\va_i$, without parameter sharing.
Recall the subscripts $\psi(\vr_{i,1:M})$ indicate a separate variational posterior for each unique set of responses.
\end{itemize}

If we compare unamortized to amortized VIBO in Fig.~\ref{fig:synth_results} (top), we see an important efficiency difference.
The number of parameters for the unamortized version scales with the number of people; the speed shows a corresponding impact, with the amortized version becoming an order of magnitude faster than the unamortized one.
In general, amortized inference is much cheaper, especially in circumstances in which the number of possible response vectors $\vr_{1:M}$ is very large (e.g. $2^{95}$ for CritLangAcq).
Comparing amortized VIBO to the un-amortized equivalent, Table~\ref{table:amortization:timing} compares the wall clock time (sec.) for the 5 real world datasets.
While VIBO is comparable to MLE and EM (Fig.~\ref{fig:realworld}a), unamortized VIBO is 2 to 15 times more expensive.

Exploring accuracy in Fig.~\ref{fig:synth_results} (bottom), we see that the unamortized variant is significantly less accurate at predicting missing data. This can be attributed to overfitting to observed responses.
With 100 items, there are $2^{100}$ possible responses from every person, meaning that even large datasets only cover a small portion of the full set.
With amortization, overfitting is more difficult as the deterministic mapping $f_\phi$ is not hardcoded to a single response vector.
Without amortization, since we learn a variational posterior for every observed response vector, we may not generalize to new response vectors.
Unamortized VIBO is thus much more sensitive to missing data as it does not get to observed the entire response.
We can see evidence of this as unamortized VIBO is superior to amortized VIBO at parameter recovery, Fig.~\ref{fig:synth_results} (middle), where no data is hidden from the model; compare this to missing data imputation, where unamortized VIBO appears inferior: because ability estimates do not share parameters, those with missing data are less constrained yielding poorer predictive performance.

Finally, when there are very few people (100) unamortized VIBO and HMC are better at recovering parameters (Fig.~\ref{fig:synth_results} middle) than amortized VIBO.
This can be explained by amortization: to train an effective regressor $f_\phi$ requires a minimum amount of data.
With too few responses, the amortization gap will be very large, leading to poor inference.
Under scarce data we would thus recommend using HMC, which is fast enough and most accurate.

\begin{table}[h!]
\caption{Time Costs with and without Amortization}
\label{table:amortization:timing}
\begin{center}
\begin{small}
\begin{sc}
\begin{tabular}{lcc}
\toprule
Dataset & Amortized (Sec.) & Un-Amortized (Sec.) \\
\midrule
CritLangAcq & 2.8k & 43.2k \\
WordBank & 176.4 & 657.1 \\
DuoLingo & 429.9 & 717.9 \\
Gradescope & 114.5 & 511.1 \\
PISA & 25.2k & 125.8k \\
\bottomrule
\end{tabular}
\end{sc}
\end{small}
\end{center}
\vskip -0.1in
\end{table}

The above suggests that amortization is important when dealing with moderate to large datasets. Turning to the structure of the amortized posteriors, we note that the factorization we chose in Thm.~\ref{thm:virtu} is only one of many.
Specifically, we could make the simpler assumption of independence between ability and item characteristics given responses in our variational posteriors: VIBO (Independent).
Such a factorization would be simpler and faster due to less gradient computation.
However, in our synthetic experiments (in which we know the true ability and item features), we found the independence assumption to produce very poor results: recovered ability and item characteristics had less than $0.1$ correlation with the true parameters.
Meanwhile the factorization we posed in Thm.~\ref{thm:virtu} consistently produced above 0.9 correlation.
Thus, the insight to decompose $q(\va_i, \vd_{1:M}|\vr_{i,1:M}) = q(\va_i|\vd_{1:M},\vr_{i,1:M})q(\vd_{1:M}|\vr_{i,1:M})$ instead of assuming independence is a critical one. (This point is also supported theoretically by research on faithful inversions of graphical models \cite{webb2018faithful}.)

\section{Deep Item Response Theory}
We have found VIBO to be fast and accurate for inference in 2PL IRT, matching HMC in accuracy and EM in speed.
This classic IRT model is a surprisingly good model for item responses despite its simplicity.
Yet it makes strong assumptions about the functional form of responses and the interaction of factors, which may not capture the nuances of human cognition.
With the advent of much larger data sets we have the opportunity to explore corrections to classic IRT models, by introducing more flexible non-linearities.
As described above, a virtue of VI is the possibility of learning aspects of the generative model by optimizing the inference objective.
We next explore several ways to incorporate learnable non-linearities in IRT, using the modern machinery of deep learning.

\subsection{Nonlinear Generalizations of IRT}
We have assumed thus far that $p(\vr_{i,1:M}|\va_i,\vd_{1:M})$ is a fixed IRT model defining the probability of correct response to each item.
We now consider three different alternatives with varying levels of expressivity that help define a class of more powerful nonlinear IRT.

\textbf{Learning a Linking Function}
We replace the logistic function in standard IRT with a nonlinear linking function.
As such, it preserves the linear relationships between items and people.
We call this VIBO (Link).
For person $i$ and item $j$, the 2PL-Link generative model is:
\begin{equation}
    p(r_{ij}|\va_i,\vd_j) =  f_\theta(-\va_i^T \vk_j - d_j)
\end{equation}
where $f_\theta$ is a one-dimensional nonlinear function followed by a sigmoid to constrain the output to be within $[0,1]$.
In practice, we parameterize $f_\theta$ as a multilayer perceptron (MLP) with three layers of 64 hidden nodes with ELU nonlinearities.

\textbf{Learning a Neural Network}$\quad$
Here, we no longer preserve the linear relationships between items and people and instead feed the ability and item characteristics directly into a neural network, which will combine the inputs nonlinearly.
We call this version VIBO (Deep).
For person $i$ and item $j$, the Deep generative model is:
\begin{equation}
    p(r_{ij}|\va_i, \vd_j) = f_\theta(\va_i,\vd_j)
\end{equation}
where again $f_\theta$ includes a Sigmoid function at the very end to preserve the correct output signatures.
This is an even more expressive model than VIBO (Link).
In practice, we parameterize $f_\theta$ as three MLPs, each with 3 layers of 64 nodes and ELU nonlinearities.
The first MLP maps ability to a hidden vector of size 64; the second maps item characteristics to a hidden vector of 64.
These two hidden vectors are concatenated and given to the final MLP, which outputs a prediction for response.

\textbf{Learning a Residual Correction}$\quad$
Although clearly a powerful model, we might fear that VIBO (Deep) becomes too uninterpretable.
So, for the third and final nonlinear model, we use the standard IRT but add a nonlinear residual component that can correct for any inaccuracies.
We call this version VIBO (Residual).
For person $i$ and item $j$, the 2PL-Residual generative model is:
\begin{equation}
    p(r_{ij}|\va_i, \vk_j, d_j) = \frac{1}{1 + e^{-\va^T_i\vk_j - d_j + f_\theta(\va_i,\vk_j,d_j)}}
\end{equation}
During optimization, we initialize the weights of the residual network to 0, thus ensuring its initial output is 0.
This encourages the model to stay close to IRT, using the residual only when necessary.
We use the same architectures for the residual component as in VIBO (Deep).

\begin{table*}[t!]
\caption{Log Likelihoods and Missing Data Imputation for Deep Generative IRT Models}
\begin{center}
\begin{small}
\begin{sc}
\begin{tabular}{lcccccc}
\toprule
Dataset & Deep IRT & VIBO (IRT-1PL) & VIBO (IRT-2PL) & VIBO (Link-2PL) & VIBO (Deep-2PL) & VIBO (Res.-2PL) \\
\midrule
CritLangAcq & - & $-11249.8 \pm 7.6$ & $-10224.0 \pm 7.1$ & $-9590.3 \pm 2.1$ & $-9311.2 \pm 5.1$ & $\mathbf{-9254.1} \pm 4.8$ \\
WordBank & - & $-17047.2 \pm 4.3$ & $-5882.5 \pm 0.8$ & $-5268.0 \pm 7.0$ & $\mathbf{-4658.4} \pm 3.9$ & $-4681.4 \pm 2.2$ \\
DuoLingo & - & $-2833.3 \pm 0.7$ & $-2488.3 \pm 1.4$ & $-1833.9 \pm 0.3$ & $-1834.2 \pm 1.3$ & $\mathbf{-1745.4} \pm 4.7$ \\
Gradescope & - & $-1090.7 \pm 2.9$ & $-876.7 \pm 3.5$ & $-750.8 \pm 0.1$ & $\mathbf{-705.1} \pm 0.5$ & $-715.3 \pm 2.7$ \\
PISA & - & $-13104.2 \pm 5.1$ & $-6169.5 \pm 4.8$ & $-6120.1 \pm 1.3$ & $-6030.2 \pm 3.3$ & $\mathbf{-5807.3} \pm 4.2$ \\
\midrule
CritLangAcq & $0.934$ & $0.927$ & $0.932$ & $0.945$ & $\mathbf{0.948}$ & $0.947$ \\
WordBank & $0.681$ & $0.876$ & $0.880$ & $0.888$ & $0.889$ & $\mathbf{0.889}$ \\
DuoLingo & $0.884$ & $0.880$ & $0.886$ & $0.891$ & $\mathbf{0.897}$ & $0.894$ \\
Gradescope & $0.813$ & $0.820$ & $0.826$ & $0.840$ & $0.847$ & $\mathbf{0.848}$ \\
PISA & $0.524$ & $0.723$ & $0.728$ & $0.718$ & $\mathbf{0.744}$ &  $0.739$ \\
\bottomrule
\end{tabular}
\end{sc}
\end{small}
\end{center}
\vskip -0.1in
\label{table:real:nonlinear}
\end{table*}


\subsection{Evaluation}
A generative model explains the data better when it assigns observations higher probability.
We thus evaluate generative models by estimating the log marginal likelihood $\log p(\vr_{1:N,1:M})$ of the training dataset.
A higher number (closer to 0) is better.
For a single person, the log marginal likelihood of his or her $M$ responses can be computed as:
\begin{equation}
    \log p(\vr_{i,1:M}) \approx \log \mathbb{E}_{q_\phi(\va_i, \vd_{1:M}|\vr_{i,1:M})}\left[ \frac{p_\theta(\vr_{i,1:M}, \va_i, \vd_{1:M})}{q_\phi(\va_i, \vd_{1:M}|\vr_{i,1:M})} \right]
    \label{eq:marg:evaluation}
\end{equation}
We use 1000 samples to estimate Eq.~\ref{eq:marg:evaluation}.

Additionally, we measure accuracy on missing data imputation as we did in section \ref{inferenceexpts}.
A more powerful generative model, that is more descriptive of the data, should also be better at filling in missing values.

\subsection{Results}
The top half of Table~\ref{table:real:nonlinear} compares the log likelihoods of observed data whereas the bottom half of Table~\ref{table:real:nonlinear} compares the accuracy of imputing missing data.
We include VIBO inference with classical IRT-1PL and IRT-2PL generative models as baselines.
We find a consistent trend: the more powerful generative models achieve a higher log likelihood (closer to 0) and a higher accuracy.
In particular, we find very large increases in log likelihood moving from IRT to Link, spanning 100 to 500 log points depending on the dataset.
Further, from Link to Deep and Residual, we find another increase of 100 to 200 log points.
In some cases, we find Residual to outperform Deep, though the two are equally parameterized, suggesting that initialization with IRT can find better local optima.
These gains in log likelihood translate to a consistent 1 to 2\% increase in held-out accuracy for Link/Deep/Residual over IRT.
This suggests that the datasets are large enough to use the added model flexibility appropriately, rather than overfitting to the data.

We also compare our deep generative IRT models with the purely deep learning approach called Deep-IRT \cite{zhang2017dynamic} (see Sec.~\ref{sec:related}), that does not model posterior uncertainty.
Unlike traditional IRT models, Deep-IRT was built for knowledge tracing and assumed sequential responses.
To make our datasets amenable to Deep-IRT, we assume an ordering of responses from $j=1$ to $j=M$.
As shown in Table~\ref{table:real:nonlinear}, our models outperform Deep-IRT in all 5 datasets by as much as 30\% in missing data imputation (e.g.~WordBank).

\subsection{Interpreting the Linking Function}
With nonlinear models, we face an unfortunate tradeoff between interpretability and expressivity.
In domains like education, practitioners  greatly value the interpretability of IRT where predictions can be directly attributed to ability or item features.
With VIBO (Deep), our most expressive model, predictions use a neural network, making it hard to understand the interactions between people and items.
\begin{figure}[h!]
    \centering
    \begin{subfigure}[b]{0.19\linewidth}
        \includegraphics[width=\linewidth]{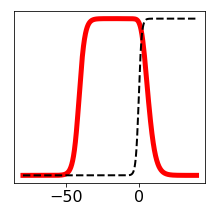}
        \caption{}
    \end{subfigure}
    \begin{subfigure}[b]{0.19\linewidth}
        \includegraphics[width=\linewidth]{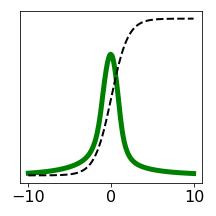}
        \caption{}
    \end{subfigure}
    \begin{subfigure}[b]{0.19\linewidth}
        \includegraphics[width=\linewidth]{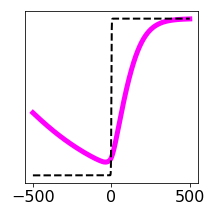}
        \caption{}
    \end{subfigure}
    \begin{subfigure}[b]{0.19\linewidth}
        \includegraphics[width=\linewidth]{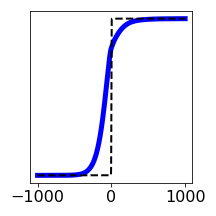}
        \caption{}
    \end{subfigure}
    \begin{subfigure}[b]{0.19\linewidth}
        \includegraphics[width=\linewidth]{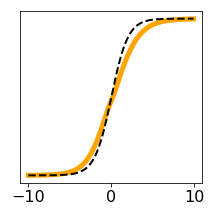}
        \caption{}
    \end{subfigure}
    \caption{Learned link functions for (a) CritLangAcq, (b) WordBank, (c) DuoLingo, (d) Gradescope, and (e) PISA.
    The dotted black line shows the default logistic function.}
    \label{fig:link}
\end{figure}

Fortunately, with VIBO (Link), we can maintain a degree of interpretability along with power.
The ``Link" generative model is identical to IRT, only differing in the linking function (i.e.~item response function).
Each subfigure in Fig.~\ref{fig:link} shows the learned response function for one of the real world datasets; the dotted black line represents the best standard linking function, a sigmoid.
We find three classes of linking functions: (1) for Gradescope and PISA, the learned function stays near a Sigmoid. (2) For WordBank and CritLangAcq, the response function closely resembles an unfolding model \cite{liu2018fitting,andrich1993hyperbolic}, which encodes a more nuanced interaction between ability and item characteristics: higher scores are related to higher ability only if the ability and item characteristics are ``nearby" in latent space.
(3) For DuoLingo, we find a piecewise function that resembles a sigmoid for positive values and a negative linear function for negative values.
In cases (2) and (3) we find much greater differences in log likelihood between VIBO (IRT) and VIBO (Link). See Table~\ref{table:real:nonlinear}.
For DuoLingo, VIBO (Link) matches the log density of more expressive models, suggesting that most of the benefit of nonlinearity is exactly in this unusual linking function.

\section{Polytomous Responses}
Thus far, we have been working only with response data collapsed into binary correct/incorrect responses.
However, many questionnaires and examinations are not binary: responses can be multiple choice (e.g. Likert scale) or even real valued (e.g. 92\% on a course test).
As we have posed IRT as a generative model, we have prescribed a Bernoulli distribution over the $i$-th person's response to the $j$-th item.
Yet nothing prevents us from choosing a different distribution, such as Categorical for multiple choice or Normal for real-values.
The DuoLingo dataset contains partial credit, computed as a fraction of times an individual gets a word correct.
A more granular treatment of these polytomous values should yield a more faithful model that can better capture the differences between people.
We thus modeled the DuoLingo data using for $p(\vr_{i,1:M}|\va_i, \vd_{1:M})$ a (truncated) Normal distribution over responses with fixed variance.
\begin{table}[h!]
\caption{DuoLingo with Polytomous Responses}
\label{table:duolingo:continuous}
\begin{center}
\begin{small}
\begin{sc}
\begin{tabular}{lcc}
\toprule
Inf. Alg. & Train & Test \\
\midrule
VIBO (IRT) &  $-22038.07$ & $-21582.03$ \\
VIBO (Link) & $-17293.35$ & $-16588.06$ \\
VIBO (Deep) & $\mathbf{-15349.84}$ & $\mathbf{-14972.66}$ \\
VIBO (Res.) & $-15350.66$ & $-14996.27$ \\
\bottomrule
\end{tabular}
\end{sc}
\end{small}
\end{center}
\vskip -0.1in
\end{table}
Table~\ref{table:duolingo:continuous} show the log densities: we again observe large improvements from nonlinear models.

Item Response Theory can in this way be extended for students who produce work where their responses go beyond binary correct/incorrect (imagine students writing text, drawing pictures, or even learning to code) encouraging educators to give students more engaging items without having to give up Bayesian response modeling.

\section{Related Work}
\label{sec:related}
We described above a variety of methods for parameter estimation in IRT such as MLE, EM, and MCMC. The benefits and drawbacks of these methods are well-documented \cite{lindenHandbookItemResponse2017}, so we need not discuss them here.
Instead, we focus specifically on methods that utilize deep neural networks or variational inference to estimate IRT parameters.


While variational inference has been suggested as a promising alternative to other inference approaches for IRT \cite{lindenHandbookItemResponse2017}, there has been surprisingly little work in this area.
In an exploration of Bayesian prior choice for IRT estimation, Natesan et al. \cite{natesan2016bayesian} posed a variational approximation to the posterior:
\begin{equation}
    p(\va_i, \vd_j|r_{i,j}) \approx q_\phi(\va_i, \vd_j) = q_\phi(\va_i)q_\phi(\vd_j)
    \label{eq:natesan}
\end{equation}
This is an unamortized and independent posterior family, unlike VIBO.
As we noted above in Sec.~\ref{ablation}, both amortization and dependence of ability on item parameters were crucial for our results.

We are aware of two approaches that incorporate deep neural networks into Item Response Theory: Deep-IRT \cite{yeung2019deep} and DIRT \cite{cheng2019dirt}. Deep-IRT is a modification of the Dynamic Key-Value Memory Network (DKVMN) \cite{zhang2017dynamic} that treats data as longitudinal, processing items one-at-a-time using a recurrent architecture. Deep-IRT produces point estimates of ability and item difficulty at each time step, which are then passed into a 1PL IRT function to produce the probability of answering the item correctly.
The main difference between DIRT and Deep-IRT is the choice of neural network: instead of the DKVMN, DIRT uses an LSTM with attention \cite{vaswani2017attention}.
In our experiments, we compare our approach to Deep-IRT and find that we outperform it by up to 30\% on the accuracy of missing response imputation.
On the other hand, our models do not capture the longitudinal aspect of response data.
Combining the two approaches would be natural.

Lastly, Curi et al. \cite{curi2019interpretable} used a Variational Autoencoder (VAE) to estimate IRT parameters in a 28-question synthetic dataset.
However, this approach only modeled ability as a latent variable, ignoring the full joint posterior incorporating item features.
In this work, our analogue to the VAE builds on the IRT graphical model, incorporating both ability and item characteristics in a principled manner.
This could explain why Curi et.~al.~report the VAE requiring substantially more data to recover the true parameters when compared to MCMC whereas we find comparable data-efficiency between VIBO and MCMC.


\section{Conclusion}
Item Response Theory is a paradigm for reasoning about the scoring of tests, surveys, and similar measurment instruments.
It is popular and plays an important role in education, medicine, and psychology.
Inferring ability and item characteristics poses a technical challenge: balancing efficiency against accuracy.
In this paper we have found that variational inference provides a potential solution, running orders of magnitude faster than MCMC algorithms while matching their state-of-the-art accuracy.
Furthermore, this approach allows us to natural extend IRT with non-linearities modeled via deep neural networks.

Many directions for future work suggest themselves.
First, further gains in speed and accuracy could be found by exploring more or less complex families of posterior approximation.
Second, more work is needed to understand deep generative IRT models and determine the most appropriate tradeoff between expressivity and interpretability.
For instance, we found significant improvements from a learned linking function, yet in some applications monotonicity may be judged important to maintain -- greater ability, for instance, should correspond to greater chance of success.
Finally, VIBO should enable more coherent, fully Bayesian, exploration of very large and important datasets, such as PISA \cite{organisation2016pisa}.

Recent advances within AI combined with new massive datasets have enabled advances in many domains.
We have given an example of this fruitful interaction for understanding humans based on their answers to questions.

%



\bibliographystyle{abbrv}
\bibliography{report}

\end{document}